\documentclass{article}

\usepackage[english]{babel}

\usepackage[letterpaper,top=2cm,bottom=2cm,left=3cm,right=3cm,marginparwidth=1.75cm]{geometry}

\usepackage{amsmath}
\usepackage{amssymb}
\usepackage{amsthm}
\usepackage{xcolor}
\usepackage{multirow}

\usepackage{graphicx}
\usepackage[colorlinks=true, allcolors=blue]{hyperref}

\usepackage{tikz}
\usepackage{tikz-3dplot}
\usetikzlibrary{calc,positioning,arrows.meta}

\theoremstyle{plain}
\newtheorem{theorem}{Theorem}
\newtheorem{lemma}[theorem]{Lemma}
\newtheorem{proposition}[theorem]{Proposition}

\theoremstyle{definition}
\newtheorem{definition}[theorem]{Definition}

\title{Neural collapse in the orthoplex regime}

\author{
James Alcala\thanks{Department of Mathematics, University of Southern California, Los Angeles, CA}
\and 
Rayna Andreeva\thanks{School of Informatics, University of Edinburgh, Edinburgh, Scotland}
\and 
Vladimir A.\ Kobzar\thanks{Department of Mathematics, The Ohio State University, Columbus, OH}
\and 
Dustin G.\ Mixon\footnotemark[3] \thanks{Translational Data Analytics Institute, The Ohio State University, Columbus, OH}
\and 
Sanghoon Na\thanks{Department of Mathematics, University of Maryland, College Park, College Park, MD}
\and 
Shashank Sule\footnotemark[5]
\and 
Yangxinyu Xie\thanks{Department of Statistics and Data Science, University of Pennsylvania, Philadelphia, PA}
}

\date{}

\begin{document}
\maketitle

\begin{abstract}
When training a neural network for classification, the feature vectors of the training set are known to collapse to the vertices of a regular simplex, provided the dimension $d$ of the feature space and the number $n$ of classes satisfies $n\leq d+1$.
This phenomenon is known as \textit{neural collapse}.
For other applications like language models, one instead takes $n\gg d$.
Here, the neural collapse phenomenon still occurs, but with different emergent geometric figures.
We characterize these geometric figures in the \textit{orthoplex regime} where $d+2\leq n\leq 2d$.
The techniques in our analysis primarily involve Radon's theorem and convexity.
\end{abstract}

\section{Introduction}

A deep neural network $N\colon\mathbb{R}^r\to\mathbb{R}^n$ can be decomposed as a feature map $F\colon\mathbb{R}^r\to\mathbb{R}^d$, followed by an affine linear map $L\colon\mathbb{R}^d\to\mathbb{R}^n$, and then the softmax function $\mathbb{R}^n\to\mathbb{R}^n$:
\[
N
\, = \,
\operatorname{softmax} \, \circ \, L \, \circ \, F.
\]
When training $N$ for a classification task, $n$ denotes the number of classes to distinguish between.
The ideal neural network maps $x\mapsto e_k$, where $e_k\in\mathbb{R}^n$ denotes the standard basis element (a.k.a.\ \textit{one-hot encoding}) supported on the intended classification $k\in[n]:=\{1,\ldots,n\}$ of the data point $x\in\mathbb{R}^r$.
To construct such a neural network, we collect a large training set of $x$'s and corresponding $k$'s before training the parameters that define $F$ and $L$ until $N(x)\approx e_k$ for each $(x,k)$ in our training set.

In 2020, Papyan, Han, and Donoho~\cite{PapyanHD:20} discovered a geometric phenomenon known as \textit{neural collapse} that emerges when training such a neural network until the training loss goes to zero.
In this ``terminal phase of training,'' several things occur simultaneously.
First, given any pair of points $x$ and $x'$ in the training set with the same intended classification $k$, it holds that the corresponding features collapse: $F(x)\approx F(x')$.
In particular, special points $h_1,\ldots,h_n$ emerge in the feature space $\mathbb{R}^d$ such that
\[
F(x)\approx h_k
\]
for every $(x,k)$ in the training set.
Furthermore, if each class is equally represented in the training set, then these special points $\{h_k\}_{k\in[n]}$ form the vertices of a \textit{regular simplex}.
(Of course, this implicitly assumes $d\geq n-1$ so that it's possible for $\mathbb{R}^d$ to contain an $n$-vertex simplex; more on that later.)
Next, if the $k$th entry of the affine linear map $L\colon\mathbb{R}^d\to\mathbb{R}^n$ is given by
\[
L(h)_k
=\langle w_k,h\rangle+b_k,
\]
then the weight vectors $\{w_k\}_{k\in[n]}$ in $\mathbb{R}^d$ that result from our training process form the vertices of a regular simplex centered at the origin such that
\[
w_k
\propto h_k-\frac{1}{n}\sum_{i=1}^n h_i.
\]
In so many words, the affine linear map $L$ ends up geometrically aligning with the feature vectors $F(x)$ of our training set.

In the time since its empirical discovery, neural collapse has been analyzed theoretically under an \textit{unconstrained features model}~\cite{MixonPP:22}.
This model decouples the analysis from the architecture of $F$ by treating the feature vectors of the training set as free variables.
(This modeling choice is justified since $F$ is highly overparameterized in practice.)
To date, almost all research in neural collapse focuses on the regime $n\le d+1$ so that the feature domain can support an $n$-vertex simplex.
But there are many use cases (e.g., language models) in which $n\gg d$.
So which geometric figures emerge when $n>d+1$?

When $d$ is fixed and $n\to\infty$, \cite{lu2022neural} established that the emergent figures are uniformly distributed on the sphere.
More recent work in \cite{JiangEtal:24} offers a nonasymptotic analysis of the $n>d+1$ regime.
This work takes $m$ examples from each class and denotes the corresponding unconstrained feature vectors by $H=\{h_{k,i}\}_{k\in[n],i\in[m]}$.
As before, the weight vectors for each class are $W=\{w_k\}_{k\in[n]}$.
Finally, we avoid the nuisance of translation and scale ambiguity by removing biases $b_k$ from $L$ and furthermore restricting each $h_{k,i}$ and $w_k$ to the unit sphere in $\mathbb{R}^d$.
Given a fixed temperature $\tau>0$, we seek to minimize the cross-entropy loss:
\[
\begin{aligned}
&\text{minimize}
\quad
&&\mathcal{L}_{\operatorname{CE}}^{(\tau)}(W,H):=
\frac{1}{mn}\sum_{k=1}^{n}\sum_{i=1}^{m} -\log \bigg( \frac{\exp(\langle w_k, h_{k,i}\rangle/\tau)  }{\sum_{j} \exp(\langle w_j, h_{k,i}\rangle/\tau)  } \bigg)\\[0.5em]
&
\text{subject to}
\quad
&&\|h_{k,i}\|=\|w_k\|=1 \quad \text{for all} ~ k\in[n],~i\in[m].
\end{aligned}
\]
In the limit as $\tau\to0$, Lemma~3.1 in~\cite{JiangEtal:24} gives that the minimizers of $\mathcal{L}_{\operatorname{CE}}^{(\tau)}$ converge to minimizers of the so-called \textit{hardmax} problem:
\[
\begin{aligned}
&\text{minimize}
\quad
&&\mathcal{L}_{\operatorname{HM}}(W,H):=
\max_{k\in [n]}\max_{i\in [m]}\max_{k'\neq k} \, \langle w_{k} - w_{k'}, h_{k,i} \rangle\\[0.5em]
&
\text{subject to}
\quad
&&\|h_{k,i}\|=\|w_k\|=1 \quad \text{for all} ~ k\in[n],~i\in[m].
\end{aligned}
\]
We call $(W,H)$ a \textbf{hardmax code} if it minimizes $\mathcal{L}_{\operatorname{HM}}(W,H)$.
In this hardmax setting, \cite{JiangEtal:24} observed a generalization of the neural collapse phenomenon in which the emergent geometric objects generalize the regular simplex.
To make this explicit, given a spherical configuration $X=\{x_k\}_{k\in[n]}$, denote
\[
\delta(X) 
:= \min_{j \in [n]} \operatorname{dist}(x_j, \operatorname{conv}\{x_i\}_{i \in [n] \setminus \{j\}}).
\]
We call $X$ a \textbf{softmax code} if it maximizes $\delta(X)$.
It turns out that for every hardmax code $(W,H)$, it holds that $W$ is a softmax code, and conversely, every softmax code $W$ can be extended to a hardmax code $(W,H)$; see Theorem~C.7 in~\cite{JiangEtal:24}.
Furthermore, one may frequently take $h_{k,i}=w_k$ for every $k\in[n]$ and $i\in[m]$ (e.g., whenever $d=2$ or $n\leq d+1$), and it is conjectured that this is always possible.

To date, softmax codes have been characterized in two settings: $d=2$ and $n\leq d+1$.
In these settings, softmax codes form the vertices of an origin-centered regular polygon or simplex, respectively.
We are particularly interested in the \textbf{orthoplex regime} in which $d+2\leq n\leq 2d$.
In this regime, it is known that any $n$ of the $2d$ points $\pm e_1,\ldots,\pm e_d\in\mathbb{R}^d$ form a softmax code.
Curiously, all of these are examples of \textit{spherical codes}, namely, solutions to the Tammes problem of finding $n$ points on the sphere that maximize the minimum pairwise distance.

In this paper, we characterize all softmax codes in the orthoplex regime.
See Figure~\ref{fig.overview} for a schematic that illustrates how our results relate to each other (and to the work of~\cite{JiangEtal:24}).
Much like the $d=2$ and $n\leq d+1$ regimes, we establish that softmax codes are precisely the spherical codes in the orthoplex regime; this characterization is the subject of Section~\ref{sec.softmax orthoplex}.
Next, in Section~\ref{sec.self duality}, we establish that for every softmax code $W$ in the orthoplex regime, taking $H=W$ delivers a hardmax code $(W,H)$.
In Section~\ref{sec.nonequal}, we observe that for positive temperatures $\tau$, certain softmax codes~$W$ have smaller values of $\mathcal{L}_{\operatorname{CE}}^{(\tau)}(W,W)$ than others, meaning not all softmax codes are created equal.
Interestingly, the best softmax codes $W$ have ``low entropy'' when $\tau$ is sufficiently small and ``high entropy'' when $\tau$ is sufficiently large.
We conclude in Section~\ref{sec.discussion} with a discussion.

\begin{figure}[t]
\centering
\begin{tikzpicture}[
  >=Latex,
  font=\small,
  node distance=10mm,
  circ/.style={
    draw,
    circle,
    align=center,
    inner sep=2.5pt,
    minimum size=22mm
  },
  lab/.style={midway, fill=white, inner sep=1.5pt, font=\scriptsize}
]

\def\R{3.5}
\coordinate (O) at (0,0);

\coordinate (v3) at ($(O)+({\R*cos(0)},{\R*sin(0)})$);
\coordinate (v4) at ($(O)+({\R*cos(288)},{\R*sin(288)})$);
\coordinate (v5) at ($(O)+({\R*cos(216)},{\R*sin(216)})$);
\coordinate (v1) at ($(O)+({\R*cos(144)},{\R*sin(144)})$);
\coordinate (v2) at ($(O)+({\R*cos(72)},{\R*sin(72)})$);

\node[circ] (c1) at (v1) {minimize\\ $\mathcal{L}_{\operatorname{CE}}^{(\tau)}$};
\node[circ] (c2) at (v2) {hardmax\\ code\\ $(W,H)$};
\node[circ] (c3) at (v3) {softmax\\ code $W$};
\node[circ] (c4) at (v4) {self-dual\\ hardmax\\ code\\ $(W,W)$};
\node[circ] (c5) at (v5) {low/high-\\entropy\\ softmax\\ codes};

\node[circ, right=18mm of c3] (c6) {spherical\\ code $W$};

\draw[->] (c1) -- (c2)
  node[lab, above, sloped, align=center, text width=14mm] {Lem.~3.1 in~\cite{JiangEtal:24}}
  node[lab, below, sloped] {$\tau\to0$};

\draw[->] (c2) -- (c3)
  node[lab, above, sloped, align=center, text width=14mm] {Thm.~C.7 in~\cite{JiangEtal:24}};

\draw[->] (c3) -- (c4)
  node[lab, below, sloped] {Thm.~\ref{thm.selfduality}};

\draw[->] (c4) -- (c5)
  node[lab, below, sloped] {Thm.~\ref{thm.low high entropy}};

\draw[->] (c1) -- (c5)
  node[lab, left] {Thm.~\ref{thm.low high entropy}~~}
  node[lab, right] {~$\tau$ small/large};

\draw[<->] (c3) -- (c6)
  node[lab, above] {Thm.~\ref{thm.softmax codes in orthoplex regime}};

\end{tikzpicture}
\caption{Schematic of our main results on neural collapse in the orthoplex regime. Theorems~\ref{thm.softmax codes in orthoplex regime}, \ref{thm.selfduality}, and \ref{thm.low high entropy} appear in Sections~\ref{sec.softmax orthoplex}, \ref{sec.self duality}, and~\ref{sec.nonequal}, respectively.\label{fig.overview}}
\end{figure}
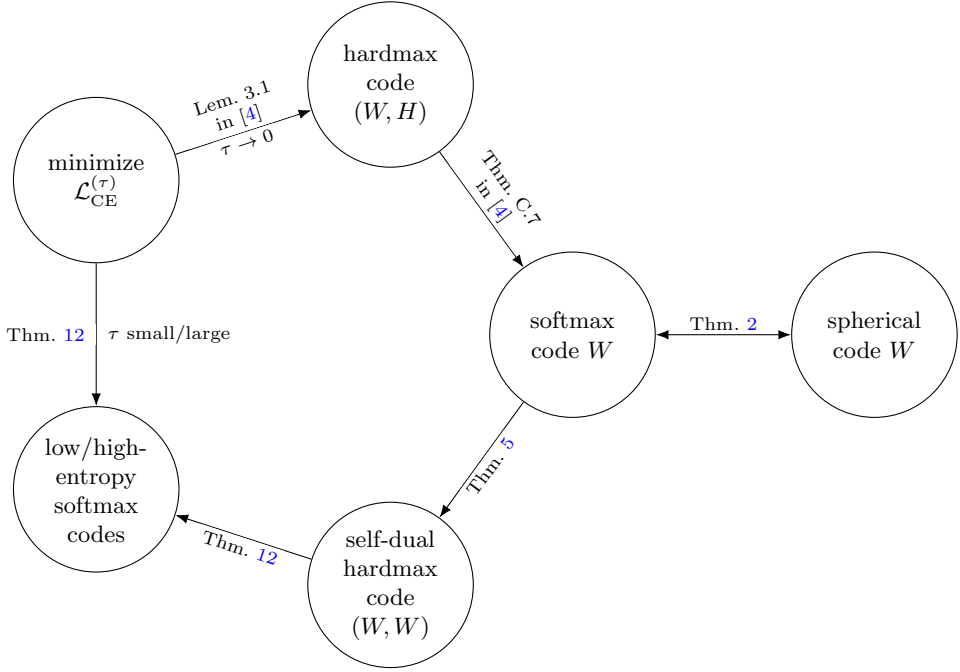

\section{Softmax codes in the orthoplex regime}
\label{sec.softmax orthoplex}

Given a spherical configuration $X = \{x_i\}_{i\in[n]}$, we denote
\[
\alpha(X)
:= \max_{\substack{i,j\in[n]\\i\neq j}}\langle x_i,x_j\rangle.
\]
We call $X$ a \textbf{spherical code} if it minimizes $\alpha(X)$.
Spherical codes have long been studied due to their application to analog coding.
Since $\alpha$ is nonconvex with many local minimizers, spherical codes are generally hard to come by, but they are fully understood in the orthoplex regime:

\begin{proposition}[Rankin's orthoplex bound; see Theorem~1 in~\cite{Rankin:55}, cf.~\cite{ORourke:mo}]
\label{prop.rankin orthoplex bound}
Fix $d\geq2$ and $n\geq d+2$.
For every configuration $X=\{x_i\}_{i\in[n]}$ in $S^{d-1}$, it holds that
\[
\alpha(X)
\geq0.
\]
Furthermore, equality is achievable if and only if $n\leq 2d$.
\end{proposition}

What follows is the main result of this section, which identifies spherical codes in the orthoplex regime with softmax codes:

\begin{theorem}
\label{thm.softmax codes in orthoplex regime}
Fix $d\geq 2$ and $n\geq d+2$.
For every configuration $X=\{x_i\}_{i\in[n]}$ in $S^{d-1}$, it holds that
\[
\delta(X)\leq 1.
\]
Furthermore, $\delta(X)=1$ if and only if $\alpha(X)=0$.
That is, in the orthoplex regime, the softmax codes are precisely the spherical codes.
\end{theorem}

Before we can prove Theorem~\ref{thm.softmax codes in orthoplex regime}, we first prove a couple of lemmas.

\begin{lemma}
\label{lem.K avoids 0}
Fix $d\geq 2$ and $n\geq d+2$.
Given a configuration $X=\{x_i\}_{i\in[n]}$ in $S^{d-1}$, suppose there exists a partition $[n]=A\sqcup B$ with both $A$ and $B$ nonempty such that the intersection
\[
K
:=\operatorname{conv}\{x_i\}_{i\in A} \cap \operatorname{conv}\{x_i\}_{i\in B}
\]
is nonempty with $0\not\in K$.
Then $\delta(X)<1$.
\end{lemma}

\begin{proof}
Let $v$ denote the projection of $0$ onto $K$.
(Notably, $v$ is nonzero since $0\not\in K$.)
Since $v\in K\subseteq\operatorname{conv}\{x_i\}_{i\in[n]}$, we may express $v$ as a convex combination:
\[
v
=\sum_{i=1}^n a_ix_i.
\]
It follows that
\[
\|v\|^2
=\langle v,v\rangle
=\bigg\langle \sum_{i=1}^n a_ix_i,v\bigg\rangle
=\sum_{i=1}^n a_i\langle x_i,v\rangle
\leq\max_{i\in[n]}\langle x_i,v\rangle,
\]
i.e., there exists $j\in[n]$ such that $\langle x_j,v\rangle\geq\|v\|^2$.
This in turn implies
\[
\|x_j-v\|^2
=1-2\langle x_j,v\rangle+\|v\|^2
\leq 1-2\|v\|^2+\|v\|^2
=1-\|v\|^2.
\]
Without loss of generality, we have $j\in A$, and so
\[
v
\in K
\subseteq \operatorname{conv}\{x_i\}_{i \in B} 
\subseteq \operatorname{conv}\{x_i\}_{i \in [n]\setminus\{j\}}.
\]
As such,
\[
\delta(X)
\leq\operatorname{dist}(x_j,\operatorname{conv}\{x_i\}_{i\in[n]\setminus\{j\}})
\leq\|x_j-v\|
\leq\sqrt{1-\|v\|^2}
<1,
\]
where the last step uses the fact that $v$ is nonzero.
\end{proof}

\begin{lemma}
\label{lem: acute angle implies short distance}
Fix $d\geq 2$ and $n \geq d + 2$. 
For every configuration $X=\{x_i\}_{i\in[n]}$ in $S^{d-1}$, it holds that 
\[
\alpha(X)
> 0
\qquad
\Longrightarrow
\qquad
\delta(X) < 1.
\]
\end{lemma}

\begin{proof}
By Radon's theorem~\cite{Radon:21}, there exists a partition $[n]=A\sqcup B$ with both $A$ and $B$ nonempty such that the intersection
\[
K
:=\operatorname{conv}\{x_i\}_{i\in A} \cap \operatorname{conv}\{x_i\}_{i\in B}
\]
is also nonempty. 
By Lemma~\ref{lem.K avoids 0}, we may assume $0\in K$.
Fix $j,k\in[n]$ such that $\langle x_j, x_k \rangle > 0$, and denote $\theta:=\arccos\langle x_j, x_k \rangle\in[0,\frac{\pi}{2})$.
Without loss of generality, we have $j \in A$, and so
\[
0 
\in K
\subseteq \operatorname{conv}\{x_i\}_{i \in B} 
\subseteq \operatorname{conv}\{x_i\}_{i \in [n]\setminus\{j\}}.
\]
Since $\operatorname{conv}\{x_i\}_{i \in [n]\setminus\{j\}}$ contains both $x_k$ and $0$, by convexity, it also contains the entire line segment $\operatorname{conv}\{x_k,0\}$.
Thus,
\[
\delta(X)
\leq \operatorname{dist}(x_j,\operatorname{conv}\{x_i\}_{i\in[n]\setminus\{j\}})
\leq \operatorname{dist}(x_j,\operatorname{conv}\{x_k,0\})
=\sin\theta
<1.
\qedhere
\]
\end{proof}

\begin{proof}[Proof of Theorem~\ref{thm.softmax codes in orthoplex regime}]
The bound $\delta(X)\leq 1$ is given by Lemma~C.9 in~\cite{JiangEtal:24}.
Meanwhile, the ``only if'' direction of the ``Furthermore'' statement follows from combining Rankin's orthoplex bound (Proposition~\ref{prop.rankin orthoplex bound}) with Lemma~\ref{lem: acute angle implies short distance}.
It remains to prove the ``if'' direction.
To this end, suppose $\alpha(X)=0$.
Then for each $j\in[n]$, it holds that every $x_i$ with $i\neq j$ resides in the halfspace
\[
H_j
:=\{z\in\mathbb{R}^d:\langle z,x_j\rangle\leq 0\}.
\]
Since $H_j$ is convex, it follows that $\operatorname{conv}\{x_i\}_{i\in[n]\setminus\{j\}}\subseteq H_j$, and so
\[
\operatorname{dist}(x_j,\operatorname{conv}\{x_i\}_{i\in[n]\setminus\{j\}})
\geq \operatorname{dist}(x_j,H_j)
= 1,
\]
where the last step follows from the projection theorem.
Minimizing over $j\in[n]$ then gives $\delta(X)\geq 1$, but since $\delta(X)\leq 1$ in general, we conclude that $\delta(X)=1$.
\end{proof}

\section{Self duality from lack of rattlers}
\label{sec.self duality}

In this section, we show that hardmax codes in the orthoplex regime enjoy a notion of \textit{self duality}:

\begin{theorem}
\label{thm.selfduality}
Every hardmax code $(W=\{w_k\}_{k\in[n]},H=\{h_{k,i}\}_{k\in[n],i\in[m]})$ in the orthoplex regime satisfies $h_{k,i}=w_k$ for every $k\in[n]$ and $i\in[m]$.
\end{theorem}

Our proof uses an important result from~\cite{JiangEtal:24}, which we state below after a requisite definition.

\begin{definition}\
\begin{itemize}
\item[(a)]
A \textbf{softmax rattler} of $X = \{x_i\}_{i\in[n]}$ is an index $j\in[n]$ such that
\[
\operatorname{dist}(x_j, \operatorname{conv}\{x_i\}_{i \in [n] \setminus \{j\}})
>\delta(X).
\]
\item[(b)]
A \textbf{Tammes rattler} of $X = \{x_i\}_{i\in[n]}$ is an index $j\in[n]$ such that
\[
\max_{i\in[n]\setminus\{j\}}\langle x_i,x_j\rangle
<\alpha(X).
\]
\end{itemize}
\end{definition}

We are now ready to reveal our proof technique for Theorem~\ref{thm.selfduality}, which factors through the following:

\begin{proposition}[Theorem~3.7 in~\cite{JiangEtal:24}]
\label{prop.selfduality from lack of rattlers}
Fix $d$ and $n$ for which every size-$n$ softmax code in $S^{d-1}$ has no softmax rattler and every size-$n$ spherical code has no Tammes rattler.
Then the following statements are equivalent:
\begin{itemize}
\item[(a)] 
Every hardmax code $(W,H)$ satisfies $h_{k,i}=w_k$ for every $k\in[n]$ and $i\in[m]$.
\item[(b)] 
The size-$n$ softmax codes in $S^{d-1}$ are precisely the size-$n$ spherical codes in $S^{d-1}$.
\end{itemize}
\end{proposition}

In~\cite{JiangEtal:24}, it was established that the softmax codes and spherical codes in both the $d=2$ and $n\leq d+1$ regimes are identical and rattle-free (in both senses), and so Proposition~\ref{prop.selfduality from lack of rattlers} gives that these codes are self dual.
We will replicate this approach to prove self duality in the orthoplex regime.
It remains to prove that these codes are rattle-free in both senses.

\begin{lemma}\label{softmax-norattler}
Every softmax code in the orthoplex regime has no softmax rattler.  
\end{lemma}

\begin{proof}
Fix a softmax code $X=\{x_i\}_{i\in[n]}$ in the orthoplex regime, meaning $\delta(X)=1$.
By Radon's theorem~\cite{Radon:21}, there exists a partition $[n]=A\sqcup B$ with both $A$ and $B$ nonempty such that the intersection
\[
K
:=\operatorname{conv}\{x_i\}_{i\in A} \cap \operatorname{conv}\{x_i\}_{i\in B}
\]
is also nonempty. 
By Lemma~\ref{lem.K avoids 0}, we have $0 \in K$. 

Consider any $j\in[n]$.
Without loss of generality, we have $j \in A$, and so 
\[
0
\in K
\subseteq \operatorname{conv}\{x_i\}_{i\in B}
\subseteq \operatorname{conv}\{x_i\}_{i\in[n]\setminus\{j\}}.
\]
It follows that
\[
1
=\delta(X)
\leq\operatorname{dist}(x_j,\operatorname{conv}\{x_i\}_{i\in[n]\setminus\{j\}} )
\leq \|x_j-0\|
=1,
\]
and so
\[
\operatorname{dist}(x_j,\operatorname{conv}\{x_i\}_{i\in[n]\setminus\{j\}} )
=\delta(X).
\]
Since $j$ was arbitrary, we conclude that $X$ has no softmax rattler.
\end{proof}

\begin{lemma}
\label{lem.no Tammes rattlers}
Every spherical code in the orthoplex regime has no Tammes rattler.
\end{lemma}

Our proof of Lemma~\ref{lem.no Tammes rattlers} makes use of a classification of spherical codes in the orthoplex regime:

\begin{proposition}[Theorem~2 in~\cite{CoxKMP:20}]
\label{prop.orthoplex structure theorem}
Fix $d\geq 2$ and $n\in[d+2,2d]$, and consider any configuration $\{x_i\}_{i\in[n]}$ in $S^{d-1}$ that achieves equality in Rankin's orthoplex bound.
Then there exists a possibly empty subset $S_0\subseteq[n]$ and a partition $S_1\sqcup\cdots\sqcup S_l=[n]\setminus S_0$ with $l\geq n-d$ such that
\begin{itemize}
\item[(a)]
$|S_0| = \operatorname{dim}\operatorname{span}\{x_i\}_{i\in S_0}$,
\item[(b)]
$|S_j| = \operatorname{dim}\operatorname{span}\{x_i\}_{i\in S_j}+1$ for each j $\in [l]$, and
\item[(c)]
$\operatorname{span}\{x_i\}_{i\in S_j} \perp \operatorname{span}\{x_i\}_{i\in S_{j'}}$ whenever $j\neq j'$.
\end{itemize}
\end{proposition}

\begin{proof}[Proof of Lemma~\ref{lem.no Tammes rattlers}]
Fix a spherical code $X=\{x_i\}_{i\in[n]}$ in the orthoplex regime, meaning $\alpha(X)=0$.
Consider any $j\in[n]$.
By Proposition~\ref{prop.orthoplex structure theorem}, there exists $i\in[n]\setminus\{j\}$ such that $\langle x_i,x_j\rangle=0=\alpha(X)$.
Since $j$ was arbitrary, we conclude that $X$ has no Tammes rattler.
\end{proof}

\begin{proof}[Proof of Theorem~\ref{thm.selfduality}]
Fix $d$ and $n$ such that $d+2\leq n\leq 2d$.
By Lemma~\ref{softmax-norattler}, every size-$n$ softmax code in $S^{d-1}$ has no softmax rattler, and by Lemma~\ref{lem.no Tammes rattlers}, every size-$n$ spherical code has no Tammes rattler.
By Theorem~\ref{thm.softmax codes in orthoplex regime}, the size-$n$ softmax codes in $S^{d-1}$ are precisely the size-$n$ spherical codes in $S^{d-1}$.
The result then follows from Proposition~\ref{prop.selfduality from lack of rattlers}.
\end{proof}

\section{Not all softmax codes are created equal}
\label{sec.nonequal}

At this point, we know all softmax codes in the orthoplex regime, and we know they are self dual, meaning they each extend uniquely to a hardmax code.
In this section, we return to the original nonzero temperature setting.
It turns out that in this setting, some softmax codes are preferred over others, and the preference is determined by the temperature.
We start with a definition.

\begin{definition}\
\begin{itemize}
\item[(a)]
A softmax code is said to be \textbf{low-entropy} with parameter $p$ if it consists of the vertices of a $p$-point regular simplex along with an orthoplex, both of which are orthogonal to each other and together span the ambient space.
\item[(b)]
A softmax code is said to be  \textbf{high-entropy} with parameter $p$ if it consists of the vertices of mutually orthogonal $p$- or $(p+1)$-point regular simplices that together span the ambient space.
\end{itemize}
\end{definition}

See Figure~\ref{fig.softmax_codes} for an illustration of low- and high-entropy softmax codes.
The vertices of every orthoplex (in which $n=2d$) are simultaneously low-entropy with $p=2$ and high-entropy with $p=2$.
Similarly, the vertices of a triangular bipyramid (in which $d=3$ and $n=5$) are simultaneously low-entropy with $p=3$ and high-entropy with $p=2$.
For an example in which these notions do not coincide, we need to go up to $d=4$ and $n=6$.
Here, a low-entropy softmax code (with $p=4$) is obtained by taking the vertices of a tetrahedron in a $3$-dimensional subspace and a pair of antipodal points in the orthogonal complement.
Meanwhile, a high-entropy softmax code (with $p=2$) is obtained by placing the vertices of an equilateral triangle in each of two orthogonal planes.

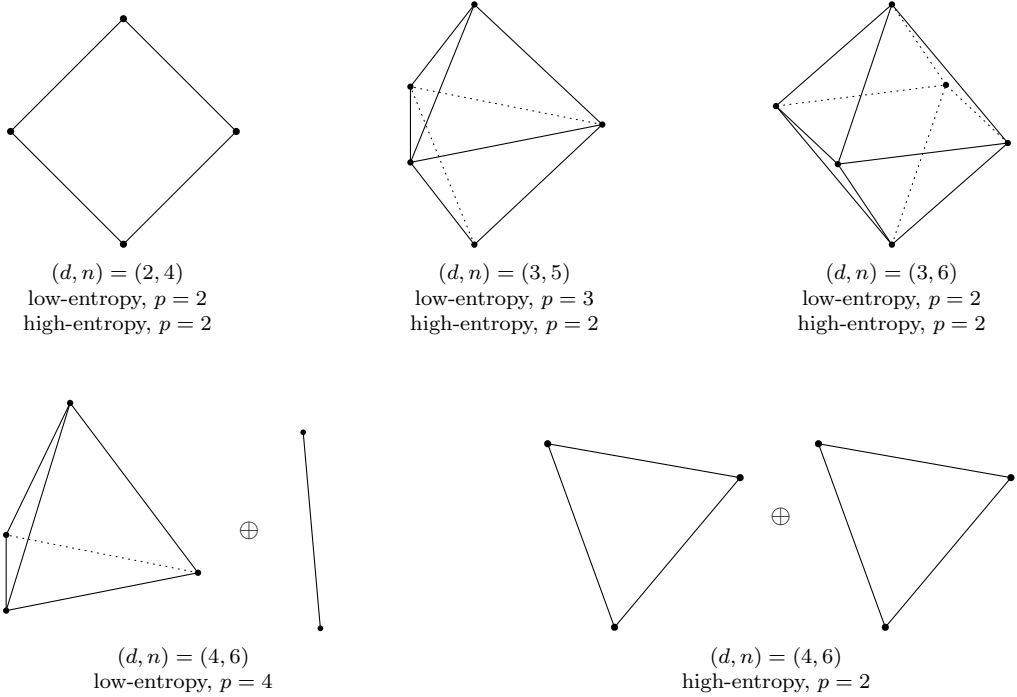
\begin{figure}[t]
\centering
\begin{tikzpicture}[scale=1.5, line cap=round, line join=round]
    \coordinate (A1) at ( 1, 0);
    \coordinate (B1) at (0, 1);
    \coordinate (C1) at (-1, 0);
    \coordinate (D1) at (0, -1);
    \draw (A1)--(B1)--(C1)--(D1)--cycle;
    \foreach \v in {A1,B1,C1,D1}{
      \fill (\v) circle (0.9pt);
    }
\end{tikzpicture}
\qquad\qquad\qquad
\begin{tikzpicture}[scale=1.7, line cap=round, line join=round]
  \tdplotsetmaincoords{70}{120}
  \begin{scope}[tdplot_main_coords]
    \coordinate (A) at ( 1, 0, 0);
    \coordinate (B) at (-0.5, {0.8660254}, 0);
    \coordinate (C) at (-0.5, {-0.8660254}, 0);

    \coordinate (T) at (0,0, 1);
    \coordinate (Btm) at (0,0,-1);

    \draw (C)--(A)--(B);
    \draw[dotted] (B)--(C);

    \foreach \v in {A,B,C}{
      \draw (T) -- (\v);
    }
    \draw (Btm) -- (A);
    \draw (Btm) -- (B);
    \draw[dotted] (Btm) -- (C);

    \foreach \v in {A,B,C,T,Btm}{
      \fill (\v) circle (0.7pt);
    }

  \end{scope}
\end{tikzpicture}
\qquad\qquad\qquad
\begin{tikzpicture}[scale=1.7, line cap=round, line join=round]
  \tdplotsetmaincoords{70}{115}
  \begin{scope}[tdplot_main_coords]
    \coordinate (xp) at ( 1, 0, 0);
    \coordinate (xm) at (-1, 0, 0);
    \coordinate (yp) at ( 0, 1, 0);
    \coordinate (ym) at ( 0,-1, 0);
    \coordinate (zp) at ( 0, 0, 1);
    \coordinate (zm) at ( 0, 0,-1);

    \foreach \a in {xp,yp,ym}{
      \draw (zp) -- (\a);
      \draw (zm) -- (\a);
    }
    \draw[dotted] (zp)--(xm);
    \draw[dotted] (zm)--(xm);
    \draw[dotted] (yp)--(xm)--(ym);
    \draw (ym)--(xp)--(yp);

    \foreach \v in {xp,xm,yp,ym,zp,zm}{
      \fill (\v) circle (0.7pt);
    }

  \end{scope}
\end{tikzpicture}

\begin{minipage}[t]{0.3\linewidth}
\centering
\footnotesize
$(d,n)=(2,4)$\\
low-entropy, $p=2$\\
high-entropy, $p=2$
\end{minipage}
\quad
\begin{minipage}[t]{0.3\linewidth}
\centering
\footnotesize
$(d,n)=(3,5)$\\
low-entropy, $p=3$\\
high-entropy, $p=2$
\end{minipage}
\quad
\begin{minipage}[t]{0.3\linewidth}
\centering
\footnotesize
$(d,n)=(3,6)$\\
low-entropy, $p=2$\\
high-entropy, $p=2$
\end{minipage}
\qquad

\vspace{24pt}

\begin{tikzpicture}[scale=1.7, line cap=round, line join=round]
  \tdplotsetmaincoords{70}{120}
  \begin{scope}[xshift=-1.4cm]
    \begin{scope}[tdplot_main_coords]
      \coordinate (A) at ( 1, 0, -1.414/4);
      \coordinate (B) at (-0.5, {0.8660254}, -1.414/4);
      \coordinate (C) at (-0.5, {-0.8660254}, -1.414/4);

      \coordinate (T) at (0,0,1.414/4*3);

      \draw (C)--(A)--(B);
      \draw[dotted] (B)--(C);
      \draw (T)--(A);
      \draw (T)--(B);
      \draw (T)--(C);

      \foreach \v in {A,B,C,T}{
        \fill (\v) circle (0.7pt);
      }

    \end{scope}
  \end{scope}

  \node at (0,0) {$\oplus$};

  \begin{scope}[scale=0.7, xshift=0.7cm, rotate=5]
    \draw (0,-1.1) -- (0,1.1);
    \fill (0, 1.1) circle (0.9pt);
    \fill (0,-1.1) circle (0.9pt);

  \end{scope}
\end{tikzpicture}
\qquad\qquad\qquad\qquad
\begin{tikzpicture}[scale=1.5, line cap=round, line join=round]
  \begin{scope}[xshift=-1.3cm, rotate=20]
    \coordinate (A1) at ( 1, 0);
    \coordinate (B1) at (-0.5, {0.8660254});
    \coordinate (C1) at (-0.5, {-0.8660254});
    \draw (A1)--(B1)--(C1)--cycle;
    \foreach \v in {A1,B1,C1}{
      \fill (\v) circle (0.9pt);
    }

  \end{scope}

  \node at (0,0) {$\oplus$};

  \begin{scope}[xshift=1.1cm, rotate=20]
    \coordinate (A2) at ( 1, 0);
    \coordinate (B2) at (-0.5, {0.8660254});
    \coordinate (C2) at (-0.5, {-0.8660254});
    \draw (A2)--(B2)--(C2)--cycle;
    \foreach \v in {A2,B2,C2}{
      \fill (\v) circle (0.9pt);
    }

  \end{scope}
\end{tikzpicture}

\begin{minipage}[t]{0.2\linewidth}
\centering
\footnotesize
$(d,n)=(4,6)$\\
low-entropy, $p=4$
\end{minipage}
\qquad\qquad\qquad\qquad\qquad\quad
\begin{minipage}[t]{0.3\linewidth}
\centering
\footnotesize
$(d,n)=(4,6)$\\
high-entropy, $p=2$
\end{minipage}

\vspace{12pt}

\caption{Smallest examples of low- and high-entropy softmax codes.\label{fig.softmax_codes}}
\end{figure}

In what follows, we let $\operatorname{SMC}(d,n)$ denote the set of all $n$-point softmax codes in $S^{d-1}$.
We are now ready to state the main result of this section.

\begin{theorem}
\label{thm.low high entropy}
Fix $d$ and $n$ such that $d+2\leq n \leq 2d$.
There exist thresholds
\[
0<\tau_-(n)<\tau_+(n)<\infty
\]
such that for each temperature $\tau<\tau_-(n)$ (resp.\ $\tau>\tau_+(n)$), the minimizers of $\mathcal{L}_{\operatorname{CE}}^{(\tau)}(W,W)$ subject to $W\in\operatorname{SMC}(d,n)$ are low- (resp.\ high-) entropy softmax codes.
\end{theorem}

By Theorem~\ref{thm.softmax codes in orthoplex regime}, we may view Proposition~\ref{prop.orthoplex structure theorem} as a decomposition of our feasibility region $\operatorname{SMC}(d,n)$.
Using the notation of Proposition~\ref{prop.orthoplex structure theorem}, we make use of the partition $S_0\sqcup S_1\sqcup\cdots\sqcup S_l=[n]$ so that the batches $\{w_k\}_{k\in S_i}$ for $i\in\{0,1,\ldots,l\}$ are mutually orthogonal.
Then the quantity to minimize is
\begin{align*}
\mathcal{L}_{\operatorname{CE}}^{(\tau)}(W,W)
&=\frac{1}{n}\sum_{k=1}^{n} -\log \bigg( \frac{\exp(\langle w_k, w_k\rangle/\tau)  }{\sum_{j} \exp(\langle w_j, w_k\rangle/\tau)  } \bigg)\\
&=\frac{1}{n}\sum_{i=0}^{l}\sum_{k\in S_i} \log \bigg( \sum_{j=1}^n \exp(\langle w_j, w_k\rangle/\tau)   \bigg)-\frac{1}{\tau}\\
&=\frac{1}{n}\sum_{i=0}^{l}\sum_{k\in S_i} \log \bigg( n-|S_i|+e^{1/\tau}+\sum_{\substack{j\in S_i\\j\neq k}}\exp(\langle w_j, w_k\rangle/\tau) \bigg)-\frac{1}{\tau}.
\end{align*}
In particular, by decomposing into mutually orthogonal batches, our objective splits batch-wise into terms of the form
\[
\mathcal{L}_{\tau,c}(X)
:=\sum_k\log\bigg(c+\sum_{j\neq k}\exp(\langle x_j, x_k\rangle/\tau)\bigg),
\]
where $c>0$ is given by $n-|S_i|+e^{1/\tau}$ when applied to the $i$th batch.
Taking inspiration from~\cite{lu2022neural}, we show that $\mathcal{L}_{\tau,c}(X)$ is minimized by the vertices of an origin-centered regular simplex.

\begin{lemma}
\label{lem.simplex batches}
Fix $d$ and $n$ such that $2\leq n \leq d+1$, along with $\tau,c>0$.
The $n$-point configurations $X=\{x_i\}_{i\in[n]}$ in $S^{d-1}$ that minimize $\mathcal{L}_{\tau,c}(X)$ satisfy
\[
\sum_{i=1}^n x_i=0
\qquad
\text{and}
\qquad
\langle x_i,x_j\rangle=-\frac{1}{n-1}
\quad \forall\, i,j\in[n],\,i\neq j.
\]
\end{lemma}

\begin{proof}
Put $s:=\sum_ix_i$ and $\beta:=1/\tau$.
Then for each $k\in[n]$, Jensen's inequality gives
\[
\frac{1}{n-1}\sum_{j\neq k}\exp(\langle x_j, x_k\rangle/\tau)
\geq \exp\bigg(\frac{1}{n-1}\sum_{j\neq k}\langle x_j, x_k\rangle/\tau\bigg)
=\exp\bigg(\frac{\beta}{n-1}\big(\langle s,x_k\rangle-1\big)\bigg).
\]
By the monotonicity of log, it follows that
\[
\mathcal{L}_{\tau,c}(X)
\geq\sum_{k=1}^n\log\bigg(c+(n-1)\exp\bigg(\frac{\beta}{n-1}\big(\langle s,x_k\rangle-1\big)\bigg)\bigg)
=\sum_{k=1}^n\log\big(c+ae^{b\langle s,x_k\rangle}\big),
\]
where $a=(n-1)e^{-\beta/(n-1)}$ and $b=\beta/(n-1)$.
Since the function $t\mapsto \log(c+ae^{bt})$ is convex and monotone, another application of Jensen's inequality gives
\[
\frac{1}{n}\mathcal{L}_{\tau,c}(X)
\geq\frac{1}{n}\sum_{k=1}^n\log\big(c+ae^{b\langle s,x_k\rangle}\big)
\geq \log\big(c+ae^{b\frac{1}{n}\sum_{k=1}^n\langle s,x_k\rangle}\big)
=\log\big(c+ae^{b\|s\|^2/n}\big)
\geq\log(c+a).
\]
We conclude by characterizing when $X$ achieves equality in this lower bound.
Equality in the last step requires $s=0$, which in turn ensures equality in the penultimate inequality.
The remaining inequality is saturated precisely when every $k$ has a constant $c_k$ such that $\langle x_j,x_k\rangle=c_k$ for every $j\neq k$.
Considering $0=\langle s,x_k\rangle=1+(n-1)c_k$, we necessarily have $c_k=-1/(n-1)$.
\end{proof}

As such, for each minimizer of $\mathcal{L}_{\operatorname{CE}}^{(\tau)}(W,W)$ subject to $W\in\operatorname{SMC}(d,n)$, the batches $S_1,\ldots,S_l$ index the vertices of origin-centered regular simplices.
It turns out that the remaining batch $S_0$ is empty, and furthermore, $l=n-d$:

\begin{lemma}
\label{lem.empty S_0}
Fix $d$ and $n$ such that $d+2\leq n \leq 2d$, along with $\tau>0$.
Consider any minimizer of $\mathcal{L}_{\operatorname{CE}}^{(\tau)}(W,W)$ subject to $W\in\operatorname{SMC}(d,n)$, and consider any partition $S_0\sqcup S_1\sqcup\cdots\sqcup S_l=[n]$ as described in Proposition~\ref{prop.orthoplex structure theorem}.
Then $S_0$ is empty, $W$ spans $\mathbb{R}^d$, and $l=n-d$.
\end{lemma}

\begin{proof}
Suppose $S_0$ is nonempty.
If $S_0$ is a singleton set, then applying Lemma~\ref{lem.simplex batches} to the spherical configuration $X:=\{w_k\}_{k\in S_0\cup S_1}$ in $\operatorname{span}X$ gives that $W$ does not minimize $\mathcal{L}_{\operatorname{CE}}^{(\tau)}(W,W)$ subject to $W\in\operatorname{SMC}(d,n)$.
Otherwise, we may apply Lemma~\ref{lem.simplex batches} to the spherical configuration $X:=\{w_k\}_{k\in S_0}$ in $\operatorname{span}X$ to reach the same conclusion.

Next, suppose $W$ does not span $\mathbb{R}^d$.
Then applying Lemma~\ref{lem.simplex batches} to the spherical configuration $X:=\{w_k\}_{k\in S_0\cup S_1}$ in $\operatorname{span}X+(\operatorname{span}W)^\perp$ gives that $W$ does not minimize $\mathcal{L}_{\operatorname{CE}}^{(\tau)}(W,W)$ subject to $W\in\operatorname{SMC}(d,n)$.

Finally, for each $i\in[l]$, let $d_i:=|S_i|-1$ denote the dimension of the span of $\{w_k\}_{k\in S_i}$.
Then
\[
d=\sum_{i=1}^l d_i
\qquad
\text{and}
\qquad
n=\sum_{i=1}^l (d_i+1)
= d+l,
\]
i.e., $l=n-d$.
\end{proof}

Overall, every minimizer consists of the vertices of mutually orthogonal origin-centered regular simplices.
Next, we determine the best tuple of simplex sizes.
For each $i\in[l]$, let $d_i:=|S_i|-1$ denote the dimension of the span of $\{w_k\}_{k\in S_i}$.
Our task is to determine the positive integers $d_1,\ldots,d_l$, all of which sum to $d$ by Lemma~\ref{lem.empty S_0}, that minimize
\begin{align*}
\mathcal{L}_{\operatorname{CE}}^{(\tau)}(W,W)
&=\frac{1}{n}\sum_{i=0}^{l}\sum_{k\in S_i} \log \bigg( n-|S_i|+e^{1/\tau}+\sum_{\substack{j\in S_i\\j\neq k}}\exp(\langle w_j, w_k\rangle/\tau) \bigg)-\frac{1}{\tau}\\
&=\frac{1}{n}\sum_{i=1}^{l}(d_i+1)\log \big( n-(d_i+1)+e^{1/\tau}+d_ie^{-1/(\tau d_i)} \big)-\frac{1}{\tau}\\
&=\frac{1}{n}\sum_{i=1}^{l}f_{n,\tau}(d_i)-\frac{1}{\tau},
\end{align*}
where the first equality is given by our original simplification of $\mathcal{L}_{\operatorname{CE}}^{(\tau)}(W,W)$, the second equality applies Lemmas~\ref{lem.simplex batches} and~\ref{lem.empty S_0}, and the last equality comes from defining $f_{n,\tau}\colon[1,n-1]\to\mathbb{R}$ by 
\[
f_{n,\tau}(x)
=(x+1)\log\big(n-x-1+e^{1/\tau}+xe^{-1/(\tau x)}\big).
\]
See Figure~\ref{fig.opt_softmax_codes_vs_temp} for an illustration of our discrete optimization problem in the special case where $n=10$.
In general, it turns out that $f_{n,\tau}$ is concave (resp.\ convex) when $\tau$ is sufficiently small (resp.\ large), which in turn dictates which choice of $(d_1,\ldots,d_l)$ minimizes $\mathcal{L}_{\operatorname{CE}}^{(\tau)}(W,W)$.

\begin{figure}[t]
\centering

\begingroup

\begin{tikzpicture}[x=12cm, y=1.2cm, line cap=round]
  \def\taumin{0.36}
  \def\taumax{0.61}
  \def\L{1}

  \pgfmathdeclarefunction{xtau}{1}{%
    \pgfmathparse{((#1)-\taumin)/(\taumax-\taumin)*\L}%
  }

  \newcommand{\drawaxis}[2]{
    \draw[->] (0,#1) -- (\L+0.05,#1);
    \node[right] at (\L+0.06,#1) {$\tau$};
    #2
  }

  \newcommand{\tick}[3]{
    \draw ({xtau(#1)},#2) -- ++(0,0.14);
    \draw ({xtau(#1)},#2) -- ++(0,-0.14);
    \node[below] at ({xtau(#1)},#2-0.14) {\scriptsize #3};
  }

  \drawaxis{0}{
    \tick{0.3916}{0}{0.3916}
    \tick{0.5847}{0}{0.5847}

    \node[above] at ({xtau((\taumin+0.3916)/2)},0) {\footnotesize concave};
    \node[above] at ({xtau((0.5847+\taumax)/2)},0) {\footnotesize convex};
  }

  \drawaxis{-1}{
    \tick{0.4968}{-1}{0.4968}

    \node[above] at ({xtau((\taumin+0.4968)/2)},-1) {\footnotesize $(3,1,1,1)$};
    \node[above] at ({xtau((0.4968+\taumax)/2)},-1) {\footnotesize $(2,2,1,1)$};
  }

  \drawaxis{-2}{
    \tick{0.4713}{-2}{0.4713}
    \tick{0.4968}{-2}{0.4968}

    \node[above] at ({xtau((\taumin+0.4713)/2)},-2) {\footnotesize $(5,1,1)$};
    \node[above] at ({xtau((0.4713+0.4968)/2)},-2) {\footnotesize $(3,3,1)$};
    \node[above] at ({xtau((0.4968+\taumax)/2)},-2) {\footnotesize $(3,2,2)$};
  }

  \drawaxis{-3}{
    \tick{0.4588}{-3}{0.4588}

    \node[above] at ({xtau((\taumin+0.4588)/2)},-3) {\footnotesize $(7,1)$};
    \node[above] at ({xtau((0.4588+\taumax)/2)},-3) {\footnotesize $(4,4)$};
  }
\end{tikzpicture}

\endgroup

\caption{Fix $n=10$. 
The first plot reports when $f_{n,\tau}$ is concave or convex.
The next three plots report the best choice of $(d_1,\ldots,d_l)$ as a function of temperature $\tau$.
Here, we take $d$ to be $6$, $7$, and~$8$, respectively.
We omit the $d=5$ case since the corresponding softmax code is unique up to rotation.
\label{fig.opt_softmax_codes_vs_temp}}
\end{figure}
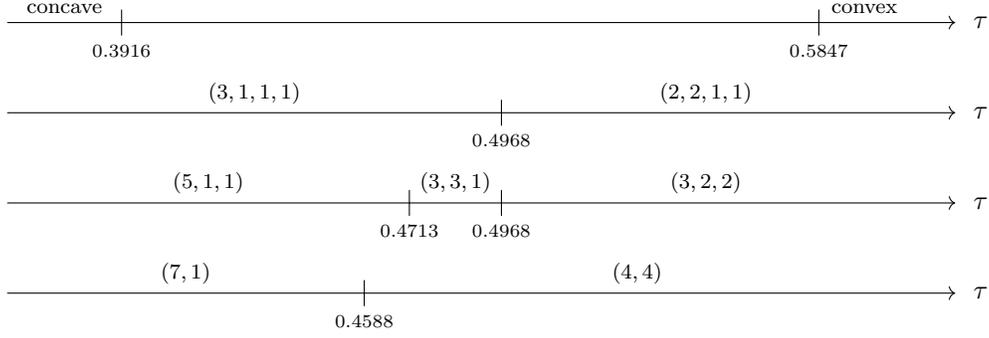

\begin{lemma}
\label{lem.concave-convex}
For every integer $n>1$, there exist thresholds 
\[
0
<\tau_-(n)
<\tau_+(n)
<\infty
\]
such that for every $\tau<\tau_-(n)$ (resp.\ $\tau>\tau_+(n)$), it holds that $f_{n,\tau}$ is strictly concave (resp.\ convex).
\end{lemma}

The proof of Lemma~\ref{lem.concave-convex} is a technical calculus argument that we defer to the end of this section.
For now, we use this lemma to prove the main result of this section.

\begin{proof}[Proof of Theorem~\ref{thm.low high entropy}]
Suppose $W$ is a minimizer of $\mathcal{L}_{\operatorname{CE}}^{(\tau)}(W,W)$ subject to $W\in\operatorname{SMC}(d,n)$.
Then as discussed above, Lemmas~\ref{lem.simplex batches} and~\ref{lem.empty S_0} imply that $W$ consists of the vertices of $l:=n-d$ mutually orthogonal origin-centered regular simplices, with the $i$th simplex having dimension $d_i$ such that $\sum_{i=1}^l d_i=d$, and furthermore,
\[
\mathcal{L}_{\operatorname{CE}}^{(\tau)}(W,W)
=\frac{1}{n}\sum_{i=1}^{l}f_{n,\tau}(d_i)-\frac{1}{\tau}.
\]
As such, it suffices to determine the tuple $(d_1,\ldots,d_l)$ of positive integers that sum to $d$ and minimize the right-hand side above.
Consider the thresholds $\tau_-(n)$ and $\tau_+(n)$ reported in Lemma~\ref{lem.concave-convex}.

\medskip
\noindent
\textbf{Case I:} $\tau<\tau_-(n)$.
We claim that the optimal $d_1\geq\cdots\geq d_l$ has $d_2=1$, and therefore the corresponding softmax code is low-entropy with parameter $p=d_1+1$.
To prove this, we will show that $d_2>1$ implies $(d_1,\ldots,d_l)$ is suboptimal.
It suffices to verify that $(d'_1,\ldots,d'_l)$ defined by $d'_1:=d_1+1$, $d'_2:=d_2-1$, and $d'_i:=d_i$ for all $i>2$ is an improvement.
First,
\[
\sum_{i=1}^l f_{n,\tau}(d_i) - \sum_{i=1}^l f_{n,\tau}(d'_i)
=\big(f_{n,\tau}(d_2)-f_{n,\tau}(d_2-1)\big)-\big(f_{n,\tau}(d_1+1)-f_{n,\tau}(d_1)\big).
\]
By the mean value theorem, there exist $c_1$ and $c_2$ satisfying
\[
d_2-1
< c_2
< d_2
\leq d_1
< c_1
< d_1+1
\]
such that 
\[
f_{n,\tau}'(c_1)=f_{n,\tau}(d_1+1)-f_{n,\tau}(d_1),
\qquad
f_{n,\tau}'(c_2)=f_{n,\tau}(d_2)-f_{n,\tau}(d_2-1).
\]
It follows that
\[
\sum_{i=1}^l f_{n,\tau}(d_i) - \sum_{i=1}^l f_{n,\tau}(d'_i)
=f'_{n,\tau}(c_2) - f'_{n,\tau}(c_1)
>0,
\]
where the last step follows from the fact that $f_{n,\tau}$ is strictly concave by Lemma~\ref{lem.concave-convex}.

\medskip
\noindent
\textbf{Case II:} $\tau>\tau_+(n)$.
We claim that the optimal $d_1\geq\cdots\geq d_l$ has $d_1-d_l\leq 1$, and therefore the corresponding softmax code is high-entropy with parameter $p=d_l+1$.
To prove this, we will show that $d_1-d_l>1$ implies $(d_1,\ldots,d_l)$ is suboptimal.
It suffices to verify that $(d'_1,\ldots,d'_l)$ defined by $d'_1:=d_1-1$, $d'_l:=d_l+1$, and $d'_i=d_i$ for all $1<i<l$ is an improvement.
First,
\[
\sum_{i=1}^l f_{n,\tau}(d_i) - \sum_{i=1}^l f_{n,\tau}(d'_i)
=\big(f_{n,\tau}(d_1)-f_{n,\tau}(d_1-1)\big)-\big(f_{n,\tau}(d_l+1)-f_{n,\tau}(d_l)\big).
\]
By the mean value theorem, there exist $c_1$ and $c_l$ satisfying
\[
d_l
<c_l
<d_l+1
\leq d_1-1
<c_1
<d_1
\]
such that
\[
f'_{n,\tau}(c_1)
=f_{n,\tau}(d_1)-f_{n,\tau}(d_1-1),
\qquad
f'_{n,\tau}(c_l)
=f_{n,\tau}(d_l+1)-f_{n,\tau}(d_l).
\]
It follows that
\[
\sum_{i=1}^l f_{n,\tau}(d_i) - \sum_{i=1}^l f_{n,\tau}(d'_i)
=f'_{n,\tau}(c_1) - f'_{n,\tau}(c_l)
>0,
\]
where the last step follows from the fact that $f_{n,\tau}$ is strictly convex by Lemma~\ref{lem.concave-convex}.
\end{proof}

\begin{proof}[Proof of Lemma~\ref{lem.concave-convex}]
For convenience, we put
\[
g(x)
:=n-x-1+e^{\beta}+xe^{-\beta/x},
\]
where $\beta:=1/\tau$.
Notably, $g(x)>0$ for every $x\in[1,n-1]$.
We have
\[
f''_{n,\tau}(x)
=\frac{1}{g(x)}\bigg((x+1)g''(x)+2g'(x)-(x+1)\frac{g'(x)^2}{g(x)}\bigg).
\]
It is sometimes convenient to write $u:=1/(\tau x)=\beta/x$.
We have
\[
g'(x)
=-1+(1+u)e^{-u},
\qquad
g''(x)
=\frac{u^2e^{-u}}{x}.
\]
In what follows, we apply the following primitives without further explanation:
\[
1\leq x\leq n-1\,,
\qquad
0\leq u\leq \beta\,.
\]

For the concavity claim, it suffices to show
\[
Q
:=(x+1)g''(x)+2g'(x)-(x+1)\frac{g'(x)^2}{g(x)}
<0.
\]
For the first two terms, we have
\begin{align*}
(x+1)g''(x)
&=(1+\tfrac{1}{x})u^2e^{-u}
\leq 2\beta^2e^{-\beta/n},\\[0.5em]
2g'(x)
&=2\big(-1+(1+u)e^{-u}\big)
\leq2\big(-1+(1+\beta)e^{-\beta/n}\big),
\end{align*}
and since
\[
g(x)
=n-x-1+e^{\beta}+xe^{-\beta/x}
\geq e^{\beta}+xe^{-\beta/x}
> 0,
\]
the third term satisfies
\[
-(x+1)\frac{g'(x)^2}{g(x)}
\leq0.
\]
Putting everything together,
\[
Q
\leq 2\big(-1+(1+\beta+\beta^2)e^{-\beta/n}\big),
\]
which is negative for all large $\beta$, i.e., it suffices to take $\tau<\tau_-(n)$ for some sufficiently small $\tau_-(n)$.

For the convexity claim, it suffices to show $Q>0$.
For the first two terms, we have
\[
(x+1)g''(x)
=(1+\tfrac{1}{x})u^2e^{-u}
\geq(1+\tfrac{1}{n})u^2e^{-u},
\qquad
2g'(x)
=2\big(-1+(1+u)e^{-u}\big).
\]
For the third term, we have
\[
g'(x)^2
=\big(-1+(1+u)e^{-u}\big)^2,
\qquad
\frac{g(x)}{x+1}
=\frac{n+e^\beta}{x+1}-1+\frac{x}{x+1}e^{-u}
\geq\frac{1}{n}+\frac{1}{2}e^{-u}.
\]
Putting everything together,
\begin{align*}
Q
&\geq(1+\tfrac{1}{n})u^2e^{-u}+2\big(-1+(1+u)e^{-u}\big)-\frac{\big(-1+(1+u)e^{-u}\big)^2}{\frac{1}{n}+\frac{1}{2}e^{-u}}\\
&=\tfrac{1}{n}u^2+(\tfrac{1}{3}+\tfrac{1}{n})u^3+\cdots,
\end{align*}
where the last step is the Taylor series expansion with respect to $u$, centered at $0$.
By an application of Taylor's theorem, it follows that $Q>0$ for all sufficiently small $u$.
Since $u=1/(\tau x)\leq 1/\tau$, it suffices to take $\tau>\tau_+(n)$ for some sufficiently large $\tau_+(n)$.
\end{proof}

\section{Discussion}
\label{sec.discussion}

In this paper, we characterized softmax codes in the orthoplex regime, we showed that they are always self dual, and we identified which of these are preferred by cross-entropy loss under extremely low and high temperatures.
Several questions remain.
First, are softmax codes always spherical codes?
This coincidence was known to hold when $d=2$ or $n\leq d+1$, and now we know it also holds when $d+2\leq n \leq 2d$.
There are only a few remaining examples of $(d,n)$ for which the corresponding spherical codes are known~\cite{EricsonZ:01}, so these are the natural next cases for further evaluation.
Is there a generalization of the theory of universal optimality~\cite{CohnK:07} that explains this coincidence?
More generally, is there any hope of determining softmax codes with $n>2d$?
Finally, it would be interesting to better understand the cross-entropy loss minimizers with positive temperatures $\tau>0$, and how the optimizers evolve with $\tau$; this is qualitatively similar to the problem investigated in~\cite{ChenGGKO:20}.

\section*{Acknowledgments}

This work was initiated at the 2024 AMS MRC workshop on ``Mathematics of Adversarial, Interpretable, and Explainable AI.''
DGM was supported by NSF DMS 2220304. 
All figures in this paper were made with the help of ChatGPT.

\end{document}